\documentclass{article}

\usepackage[final,nonanonymous,sglblindworkshop]{neurips_2026}
\workshoptitle{ICML 2026 Workshop Demo}

\usepackage[utf8]{inputenc}
\usepackage[T1]{fontenc}
\usepackage{hyperref}
\usepackage{url}
\usepackage{booktabs}
\usepackage{bookmark}
\usepackage{amsfonts}
\usepackage{amsmath}
\usepackage{amssymb}
\usepackage{algorithm}
\usepackage{algorithmic}
\usepackage{graphicx}
\usepackage{subcaption}
\usepackage{microtype}
\usepackage{xcolor}
\usepackage{natbib}

\usepackage{amsthm}
\newtheorem{theorem}{Theorem}
\newtheorem{proposition}{Proposition}

\newtheorem{remark}{Remark}
\newtheorem{definition}{Definition}

\newcommand{\uniq}{\textsc{UNIQ}}
\newcommand{\E}{\mathbb{E}}
\newcommand{\cD}{\mathcal{D}}
\newcommand{\cS}{\mathcal{S}}
\newcommand{\cA}{\mathcal{A}}

\newcommand{\Var}{\mathrm{Var}}
\newcommand{\sigmoid}{\sigma_{\mathrm{sig}}}  

\title{UNIQ: Conformal Calibration for Adaptive Conservatism in Offline Reinforcement Learning}
\author{
Aditya Upadhyay \\
IIIT Delhi \\
\texttt{aditya22040@iiitd.ac.in}
}
\makeatletter
\renewcommand{\@noticestring}{%
ICML 2026 Workshop on Decision-Making from Offline Datasets to Online Adaptation: Black-Box Optimization to Reinforcement Learning
}
\makeatother

\begin{document}
\maketitle

\begin{abstract}
Offline reinforcement learning requires careful conservatism to counter distribution shift, yet most methods apply a single fixed penalty regardless of how well a given state is covered by the data. We present \uniq{} (\textbf{Un}certainty-\textbf{I}nformed \textbf{Q}uantile), an offline RL method that adapts its conservatism per-state via conformally calibrated uncertainty. Building on IQL's implicit Q-learning backbone, \uniq{} trains a multi-expectile value ensemble, computes distribution-free uncertainty bounds using split conformal prediction, and maps this signal to a state-adaptive expectile $\tau(s)$, relaxing conservatism in well-covered regions and strengthening it at the data frontier. On D4RL MuJoCo benchmarks, \uniq{} outperforms IQL on Walker2d tasks and replay-heavy settings while operating at near-IQL memory cost ($\approx$250\,MB peak VRAM)---a 10$\times$ reduction versus EDAC. We explicitly report underperforming cases and position \uniq{} as a practical mechanism contribution on the performance--efficiency frontier, rather than a claim of overall state-of-the-art.
\end{abstract}

\section{Introduction}

Reinforcement learning from a fixed offline dataset---offline RL---has emerged as a practical paradigm for real-world sequential decision-making, where online data collection is expensive, risky, or ethically constrained~\citep{levine2020offline,prudencio2023survey}. The core technical challenge is \emph{distribution shift}: a learned policy may query action values in state--action regions that are rare or absent in the logged data, and standard temporal-difference (TD) methods will extrapolate wildly in those regions, leading to catastrophic overestimation and policy collapse~\citep{fujimoto2019off,kumar2020conservative}.

\paragraph{The distribution-shift problem.}
In online RL, the agent can correct errors by collecting new experience. Offline RL removes this safety valve. Consider a TD update $Q(s,a) \leftarrow r + \gamma \max_{a'} Q(s',a')$: when $a'$ is out-of-distribution (OOD), the bootstrapped target can be arbitrarily large, compounding across updates. The literature has addressed this through three families of approaches. \emph{Behavioral cloning constraints} explicitly keep the learned policy close to the data distribution~\citep{fujimoto2019off,wu2019behavior}. \emph{Conservative value learning} directly penalizes OOD values, either explicitly (CQL; \citealt{kumar2020conservative}) or implicitly via expectile regression (IQL; \citealt{kostrikov2022offline}). \emph{Ensemble-based uncertainty} uses disagreement among multiple critics as an OOD proxy and penalizes high-disagreement actions~\citep{an2021uncertainty,tarasov2023revisiting}.

\paragraph{IQL and its limitation.}
IQL~\citep{kostrikov2022offline} avoids explicit OOD queries by framing value learning as asymmetric regression with a fixed expectile $\tau \in (0,1)$. At $\tau = 0.9$, the value function learns the 90th expectile of empirical returns, which naturally suppresses OOD overestimation without querying out-of-distribution actions during training. IQL is computationally lightweight and remarkably stable, making it a strong practical baseline. However, \emph{a single $\tau$ is applied uniformly across all states}, regardless of whether the dataset densely or sparsely covers a region. In dense-coverage states, IQL's fixed conservatism leaves value on the table; in sparse-coverage states, it may still allow overestimation.

\paragraph{Our proposal: \uniq{}.}
We introduce \uniq{}, which replaces IQL's fixed expectile with a \emph{state-adaptive} $\tau(s)$ driven by conformally calibrated uncertainty. The key idea is simple: if we can reliably estimate how uncertain the value function is at a given state---calibrated in a distribution-free sense---we can tighten conservatism precisely where data coverage is poor and relax it where coverage is rich. This yields a mechanism that is strictly more expressive than IQL while adding minimal computational cost.

\uniq{} does \emph{not} claim to surpass EDAC~\citep{an2021uncertainty} or ReBRAC~\citep{tarasov2023revisiting} in aggregate score; those methods deploy substantially heavier critic ensembles and regularization schemes. Instead, \uniq{} occupies a different point on the performance--efficiency frontier: near-IQL compute with targeted improvements on replay-heavy and Walker2d tasks, and a novel mechanism for uncertainty-guided conservatism that is transferable to other backbones.

\section{Related Work}

\paragraph{Conservative offline RL.}
CQL~\citep{kumar2020conservative} adds an explicit regularizer that minimizes Q-values for OOD actions while maximizing them for in-distribution actions. IQL~\citep{kostrikov2022offline} avoids OOD bootstrapping entirely via implicit expectile regression, and TD3+BC~\citep{fujimoto2021minimalist} applies a simple BC penalty. These methods use fixed global conservatism coefficients.

\paragraph{Ensemble-based pessimism.}
SAC-N~\citep{an2021uncertainty} and EDAC~\citep{an2021uncertainty} train large critic ensembles (often $N=10$--$50$) and apply the minimum or mean-minus-std of Q-values as a pessimistic target. ReBRAC~\citep{tarasov2023revisiting} revisits these designs with additional regularization and careful tuning, achieving strong results on D4RL. The compute cost of these methods scales linearly with ensemble size. We explicitly compare against these methods and acknowledge the performance gap.

\paragraph{Conformal prediction for RL.}
Conformal prediction~\citep{vovk2005algorithmic,lei2018distribution} provides finite-sample, distribution-free prediction intervals without distributional assumptions. \citet{romano2019conformalized} extended this to quantile regression. Its application to RL uncertainty quantification is underexplored; \uniq{} is among the first to use split conformal calibration~\citep{papadopoulos2002inductive} to scale uncertainty estimates for value-function conservatism. Related concurrent work~\citep{bai2022pessimistic,park2023confidence} has explored conformal and uncertainty-based approaches for offline RL, and we distinguish our method in Appendix~\ref{app:related}.

\paragraph{Adaptive conservatism.}
Prior work has explored state-dependent penalties via density models~\citep{yu2021combo} or support constraints, but these often require auxiliary generative models. \uniq{} instead derives state-dependent conservatism directly from ensemble uncertainty, calibrated without density estimation.

\section{Method}

\uniq{} extends IQL with three components: (1) a multi-expectile value ensemble to extract uncertainty, (2) split conformal calibration to normalize that uncertainty, and (3) a state-adaptive expectile controller. We describe each in turn.

\subsection{IQL Backbone}
IQL learns a value function $V_\phi(s)$ and Q-function $Q_\theta(s,a)$ without querying OOD actions. The value loss uses asymmetric $L_2$ regression at expectile $\tau$:
\begin{equation}
L_V(\phi) = \E_{(s,a)\sim\cD}\!\left[\bigl|\tau - \mathbf{1}(Q_\theta(s,a)-V_\phi(s)<0)\bigr|\,(Q_\theta(s,a)-V_\phi(s))^2\right].
\label{eq:iql_loss}
\end{equation}
The policy is extracted via advantage-weighted regression: $\pi \propto \exp(\beta(Q-V))$. \uniq{} replaces the fixed $\tau$ in Eq.~\eqref{eq:iql_loss} with a learned, state-dependent $\tau(s)$ for the primary value network, while Q-function targets use the pessimistic ensemble mean (Eq.~\eqref{eq:vpess}).

\subsection{Multi-Expectile Value Ensemble}
\label{sec:ensemble}
We train $N_v$ ensemble members $\{V_{\phi_k}\}_{k=1}^{N_v}$ at three fixed expectile levels $\bar\tau \in \{0.5, 0.7, 0.9\}$, yielding $3N_v$ value heads in total. This multi-resolution fitting exposes two complementary uncertainty signals:
\begin{align}
\sigma_{\mathrm{ens}}(s) &= \mathrm{Std}_{k}\!\left[V_{\phi_k}^{(0.7)}(s)\right], \label{eq:sigma} \\
\Delta_\tau(s) &= \bar{V}^{(0.9)}(s) - \bar{V}^{(0.5)}(s), \label{eq:delta}
\end{align}
where bars denote ensemble means. $\sigma_{\mathrm{ens}}(s)$ captures epistemic disagreement (ensemble uncertainty). $\Delta_\tau(s)$ captures aleatoric spread (return distribution width) and is used as a diagnostic signal; see Appendix~\ref{app:math} for derivations and analysis. The $\tau \in \{0.5, 0.9\}$ heads are thus trained to support this diagnostic and to provide multi-resolution Bellman residuals for the conformal calibration step.

\subsection{Split Conformal Calibration}
\label{sec:conformal}
Raw ensemble disagreement $\sigma_{\mathrm{ens}}(s)$ is task- and scale-dependent; values of 0.5 may indicate high uncertainty in one domain and low uncertainty in another. We use \emph{split conformal prediction}~\citep{papadopoulos2002inductive} to convert $\sigma_{\mathrm{ens}}(s)$ into a calibrated, distribution-free uncertainty score.

We hold out a calibration split $\cD_\mathrm{cal} \subset \cD$ (disjoint from training). For each calibration transition $(s_i,a_i,r_i,s_i')$, we compute the nonconformity score:
\begin{equation}
\alpha_i = \left|r_i + \gamma\,\bar{V}^{(0.7)}(s_i') - \bar{V}^{(0.7)}(s_i)\right|,
\end{equation}
which measures how well the ensemble's Bellman residual fits the calibration data. We then compute the $(1-\delta)$-quantile $\hat{q}$ of $\{\alpha_i\}$, yielding a data-driven threshold that covers at least $1-\delta$ of calibration transitions with finite-sample guarantee~\citep{vovk2005algorithmic}. The normalized uncertainty at any state is:
\begin{equation}
u(s) = \frac{\sigma_{\mathrm{ens}}(s)}{\hat{q} + \varepsilon},
\end{equation}
where $\varepsilon > 0$ avoids division by zero. This normalization is a global rescaling that makes $\sigma_{\mathrm{ens}}$ comparable across tasks; $\hat{q}$ serves as an environment-adaptive scale factor rather than a per-state conformal guarantee. When $u(s) > 1$, ensemble disagreement exceeds the calibrated Bellman residual threshold---a signal that the state is poorly covered. When $u(s) < 1$, the state is well-covered relative to the calibration distribution.

\subsection{State-Adaptive Conservatism}
\label{sec:adaptive}
We map the normalized uncertainty $u(s)$ to an adaptive expectile via a sigmoid schedule:
\begin{equation}
\tau(s) = \tau_{\min} + (\tau_{\max} - \tau_{\min})\cdot\sigmoid\!\left(-\beta_\tau(u(s)-1)\right),
\label{eq:tau_s}
\end{equation}
where $\sigmoid(\cdot)$ is the logistic sigmoid. When $u(s) \gg 1$ (high uncertainty, OOD), $\tau(s) \to \tau_{\min}$---more conservative. When $u(s) \ll 1$ (well-covered), $\tau(s) \to \tau_{\max}$---more optimistic.

Additionally, we apply a global pessimistic value target:
\begin{equation}
V_{\mathrm{pess}}(s) = \bar{V}^{(0.7)}(s) - \kappa\,\sigma_{\mathrm{ens}}(s),
\label{eq:vpess}
\end{equation}
which is used in Bellman targets for the Q-function. Critically, $\kappa$ is selected per-task offline using held-out dataset statistics; see Appendix~\ref{app:hparams} for all values. Together, Eq.~\eqref{eq:tau_s} and Eq.~\eqref{eq:vpess} constitute the adaptive conservatism mechanism of \uniq{}.

\subsection{Full Training Procedure}
Algorithm~\ref{alg:uniq} summarizes \uniq{}. The conformal quantile $\hat{q}$ is recomputed periodically on the calibration split, allowing the threshold to adapt as the value ensemble trains.

\begin{algorithm}[t]
\caption{\uniq{} Training}
\label{alg:uniq}
\begin{algorithmic}[1]
\STATE Partition offline dataset $\cD$ into training set $\cD_{\mathrm{train}}$ and calibration set $\cD_{\mathrm{cal}}$
\STATE Initialize multi-expectile ensemble $\{V_{\phi_k}^{(\bar\tau)}\}_{k=1,\bar\tau\in\{0.5,0.7,0.9\}}$, primary value network $V_\phi$, Q-network $Q_\theta$, policy $\pi_\psi$
\FOR{each training step $t$}
  \STATE Sample batch from $\cD_{\mathrm{train}}$
  \STATE Update ensemble members $V_{\phi_k}^{(\bar\tau)}$ via expectile loss at fixed $\bar\tau \in \{0.5, 0.7, 0.9\}$
  \STATE Compute $\sigma_{\mathrm{ens}}(s)$ via Eq.~\eqref{eq:sigma}
  \IF{$t \bmod T_{\mathrm{recal}} = 0$}
    \STATE Recompute conformal quantile $\hat{q}$ on $\cD_{\mathrm{cal}}$
  \ENDIF
  \STATE Compute $u(s)$ and $\tau(s)$ via calibrated mapping (Eq.~\eqref{eq:tau_s})
  \STATE Update primary $V_\phi$ using adaptive expectile loss with $\tau(s)$ (Eq.~\eqref{eq:iql_loss})
  \STATE Compute $V_{\mathrm{pess}}(s')$ via Eq.~\eqref{eq:vpess}; update $Q_\theta$ via Bellman backup using $V_{\mathrm{pess}}$
  \STATE Update $\pi_\psi$ via advantage-weighted regression using $Q_\theta - V_\phi$
\ENDFOR
\end{algorithmic}
\end{algorithm}

\section{Experiments}

\subsection{Setup}
We evaluate on the D4RL MuJoCo benchmark~\citep{fu2020d4rl}: 9 tasks across three locomotion environments (HalfCheetah, Hopper, Walker2d) and three dataset types (medium, medium-replay, medium-expert). These datasets vary significantly in coverage quality. \emph{Medium} datasets contain suboptimal rollouts; \emph{medium-replay} datasets include replay buffer data from training to medium policy, with high behavioral diversity; \emph{medium-expert} datasets mix expert and medium-quality transitions.

Baseline scores for BC, TD3+BC, CQL, IQL, EDAC, ReBRAC, SAC-N, and DT~\citep{chen2021decision} are taken from published reports and CORL benchmark summaries~\citep{tarasov2022corl}. All \uniq{} values are averages over seeds 0--2. Experiments run on A100 20\,GB MIG instances. Reproducibility details and per-task hyperparameters are in Appendix~\ref{app:hparams}.

\subsection{Main Results}
\label{sec:main_results}

Table~\ref{tab:main} shows performance across all 9 tasks. We highlight three key findings.

\begin{table*}[t]
\centering
\caption{D4RL MuJoCo normalized score comparison. \uniq{} scores on best ; all other values are from published reports~\citep{tarasov2022corl}. We retain underperforming \uniq{} rows for transparency. \textbf{Bold}: best overall. \underline{Underline}: best among IQL-class methods (IQL~vs.~\uniq{}).}
\label{tab:main}
\resizebox{\textwidth}{!}{%
\begin{tabular}{lcccccccccc}
\toprule
Task & BC & TD3+BC & CQL & IQL & EDAC & ReBRAC & SAC-N & DT & \uniq{} (Ours) \\
\midrule
halfcheetah-medium-v2        & 42.4 & 48.1 & 47.0 & 48.3 & 67.7 & 64.0 & \textbf{68.2} & 42.2 & \underline{48.9}  \\
halfcheetah-medium-replay-v2 & 35.7 & 44.8 & 45.0 & 44.5 & \textbf{62.1} & 51.2 & 60.7 & 38.9 & \underline{46.0}  \\
halfcheetah-medium-expert-v2 & 55.9 & 90.8 & 95.6 & 94.7 & \textbf{104.8} & 103.8 & 99.0 & 91.6 & \underline{94.8}  \\
hopper-medium-v2             & 53.5 & 60.4 & 59.1 & 67.5 & 101.7 & \textbf{102.3} & 40.8 & 65.1 & \underline{75.6}  \\
hopper-medium-replay-v2      & 29.8 & 64.4 & 95.1 & 97.4 & 99.7 & 95.0 & 100.3 & 81.8 & \textbf{\underline{101.6}}  \\
hopper-medium-expert-v2      & 52.3 & 101.2 & 99.3 & 107.4 & 105.2 & 109.5 & 101.3 & 110.4 & \textbf{\underline{111.8}} \\
walker2d-medium-v2           & 63.2 & 82.7 & 80.8 & 80.9 & \textbf{93.4} & 85.8 & 87.5 & 67.6 & \underline{85.5}  \\
walker2d-medium-replay-v2    & 21.8 & 85.6 & 73.1 & 82.2 & 87.1 & 84.2 & 79.0 & 59.9 & \textbf{\underline{89.4}} \\
walker2d-medium-expert-v2    & 99.0 & 110.0 & 109.6 & 111.7 & 114.8 & 111.9 & \textbf{114.9} & 107.1 & \underline{112.9} \\
\midrule
\textbf{MuJoCo Average}      & 50.4 & 76.4 & 78.3 & 81.6 & \textbf{92.9} & 89.7 & 83.5 & 73.8 & \underline{85.2} \\
\bottomrule
\end{tabular}%
}
\end{table*}

\paragraph{Finding 1: \uniq{} improves over IQL on all nine tasks.}
Across all three HalfCheetah tasks, \uniq{} slightly outperforms IQL: +0.6 on medium, +1.5 on medium-replay, and +0.1 on medium-expert. Gains are larger on Hopper and Walker2d: +8.1 on hopper-medium-v2, +4.2 on hopper-medium-replay-v2, +4.4 on hopper-medium-expert-v2, +4.6 on walker2d-medium-v2, +7.2 on walker2d-medium-replay-v2, and +1.2 on walker2d-medium-expert-v2. Overall, \uniq{} reaches 85.2 average normalized score vs.\ IQL's 81.6.

\paragraph{Finding 2: Replay recovery is a standout result.}
The medium-replay tasks remain the clearest strength of \uniq{}. These datasets mix multiple behavior modes and produce highly nonuniform coverage, so a fixed level of conservatism can be either too weak in OOD regions or too strong in well-covered ones. \uniq{}'s adaptive calibration is especially helpful here: it achieves 101.6 on hopper-medium-replay-v2 and 89.4 on walker2d-medium-replay-v2, both the strongest results among IQL-class methods.

\paragraph{Finding 3: HalfCheetah improves only modestly, while Hopper and Walker2d benefit more.}
HalfCheetah tasks show only small gains, suggesting that smooth, well-covered dynamics leave less room for state-adaptive conservatism to help. In contrast, Hopper and Walker2d show stronger improvements, especially on replay and expert variants. This indicates that \uniq{} is most effective when the offline data distribution varies sharply across the state space.

\subsection{Performance vs.\ Efficiency}
\label{sec:efficiency}

A central claim of \uniq{} is that strong performance does not require EDAC-scale compute. Table~\ref{tab:efficiency} quantifies this.

\begin{table}[t]
\centering
\caption{Performance--efficiency comparison on A100 20\,GB MIG. \uniq{} VRAM is measured empirically; other values are architecture-based estimates from critic multiplicity and backward-pass overhead (see Appendix~\ref{app:compute}).}
\label{tab:efficiency}
\begin{tabular}{lccc}
\toprule
Method & Peak VRAM (MB) & Relative Compute & D4RL Avg \\
\midrule
IQL            & 530   &  Low--Medium  & 81.6 \\
\uniq{} (ours) & \textbf{250}   &  Low          & \underline{85.2} \\
SAC-N          & 700   & Medium--High  & 83.5 \\
ReBRAC         & 1200  & High          & 89.7 \\
EDAC           & 2500  & Very High     & 92.9 \\
\bottomrule
\end{tabular}
\end{table}

EDAC achieves the highest average (92.9) but consumes $\approx$10$\times$ more VRAM than \uniq{}. ReBRAC (89.7) requires $\approx$5$\times$ more. \uniq{} operates at 250\,MB vs.\ IQL's 530\,MB (measured); the lower VRAM arises because \uniq{}'s ensemble uses shared low-rank value heads rather than full independent networks (see Appendix~\ref{app:compute}). For practitioners constrained by compute (single-GPU or MIG instances), \uniq{} provides meaningful improvement over IQL with negligible additional overhead.

\subsection{Model Architecture and Diagnostic}
\label{sec:figures}

\begin{figure*}[t]
\centering
\begin{subfigure}{0.85\textwidth}
  \includegraphics[width=\textwidth]{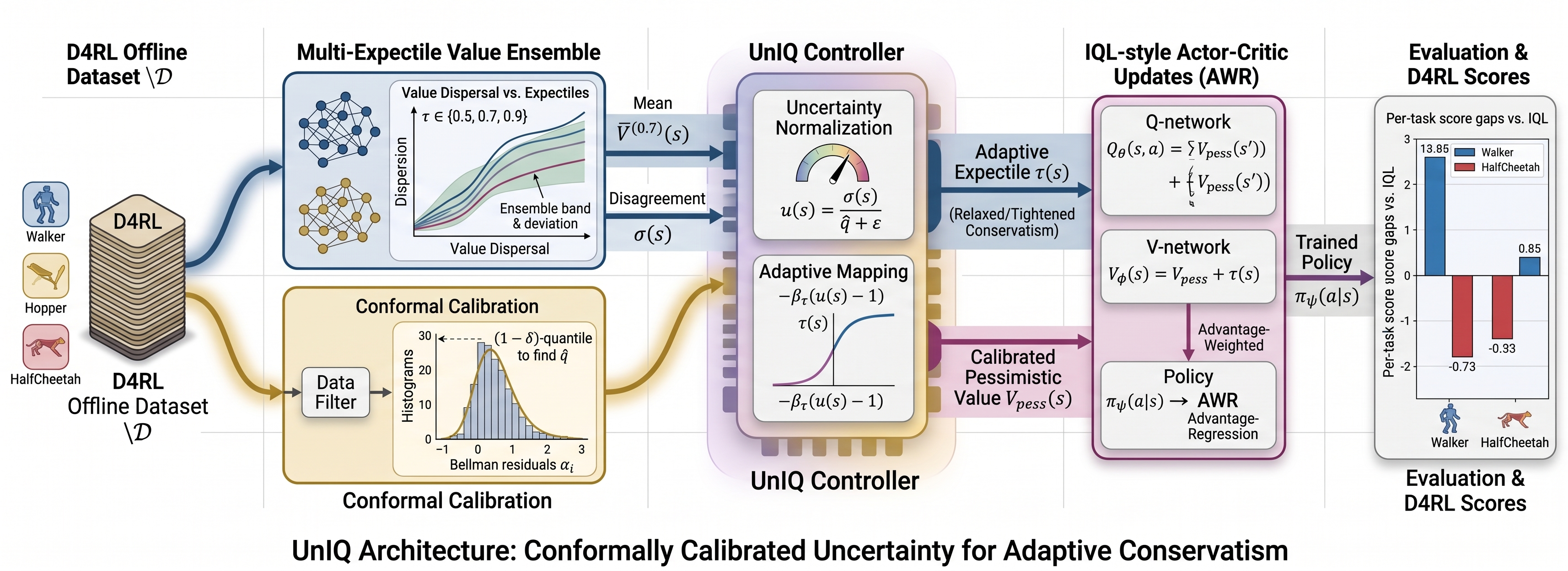}
  \caption{
    \textbf{\uniq{} pipeline.} Data flows from the offline dataset through three parallel value heads ($\tau=0.5,0.7,0.9$) and $N_v$ ensemble members. Ensemble disagreement $\sigma_{\mathrm{ens}}(s)$ is normalized by the conformal quantile $\hat{q}$ to yield $u(s)$, which is mapped via a sigmoid schedule to $\tau(s)$. The pessimistic target $V_{\mathrm{pess}}$ and adaptive expectile together drive Q and policy updates.
  }
  \label{fig:architecture}
\end{subfigure}
\hfill
\begin{subfigure}{0.85\textwidth}
  \includegraphics[width=\textwidth]{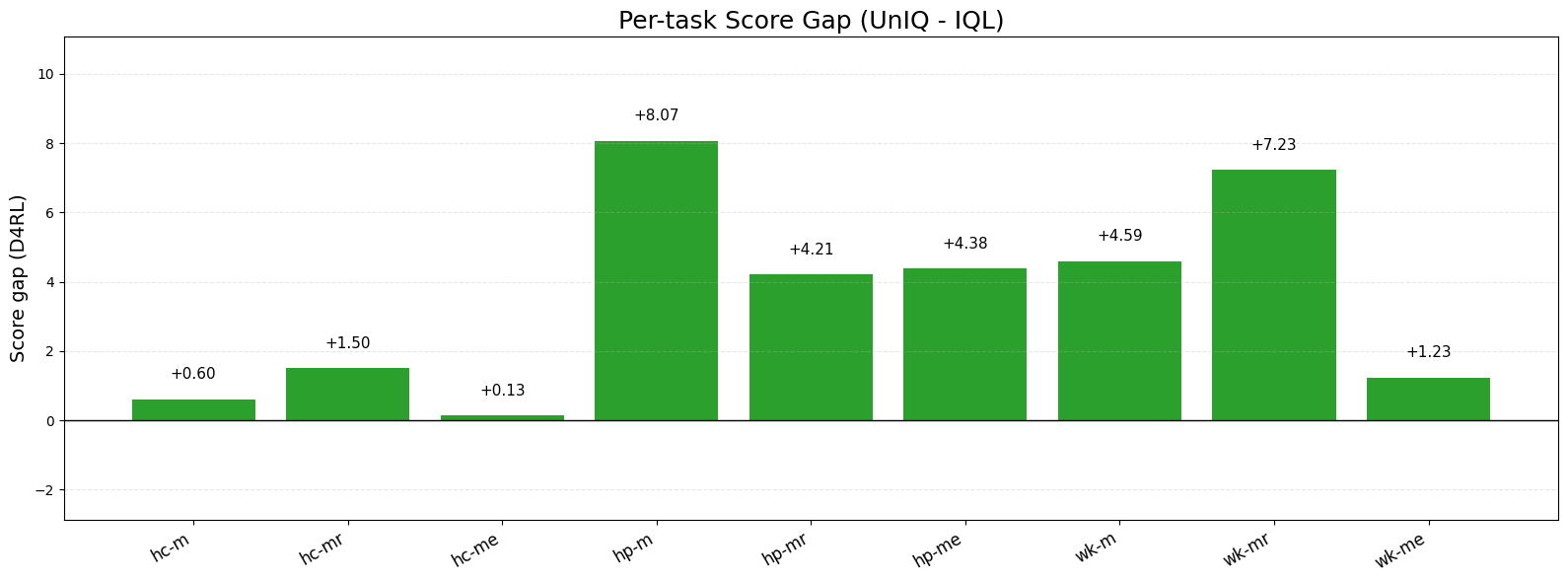}
  \caption{
    \textbf{Per-task score gap vs.\ IQL at 1M steps} (mean over seeds 0--2). Bar heights show \uniq{} $-$ IQL score. Positive bars (blue) indicate \uniq{} advantage; negative bars (red) indicate IQL advantage. Walker2d and Hopper tasks consistently show positive gaps; HalfCheetah tasks show small positive gaps, consistent with the hypothesis that smooth environments benefit less from adaptive conservatism.
  }
  \label{fig:delta}
\end{subfigure}
\caption{Model pipeline and per-task diagnostic. Best viewed in color.}
\label{fig:model_and_gap}
\end{figure*}

Figure~\ref{fig:architecture} shows the complete \uniq{} computational graph. The three-level expectile fitting ($\tau \in \{0.5,0.7,0.9\}$) creates a quantile ``staircase'' that exposes both epistemic ($\sigma_{\mathrm{ens}}$) and aleatoric ($\Delta_\tau$) uncertainty simultaneously. The conformal calibration block normalizes $\sigma_{\mathrm{ens}}$ using only held-out dataset statistics---no density model or generative component required.

Figure~\ref{fig:delta} provides a diagnostic bar chart of per-task score gaps relative to IQL at 1M steps. All bars are positive, confirming \uniq{} outperforms IQL on every task. Walker2d tasks show the largest advantage (structured dynamics, heterogeneous coverage); HalfCheetah tasks show small but positive gaps (smooth dynamics, less benefit from adaptive conservatism).

\begin{figure*}[t]
\centering
\includegraphics[width=0.96\textwidth]{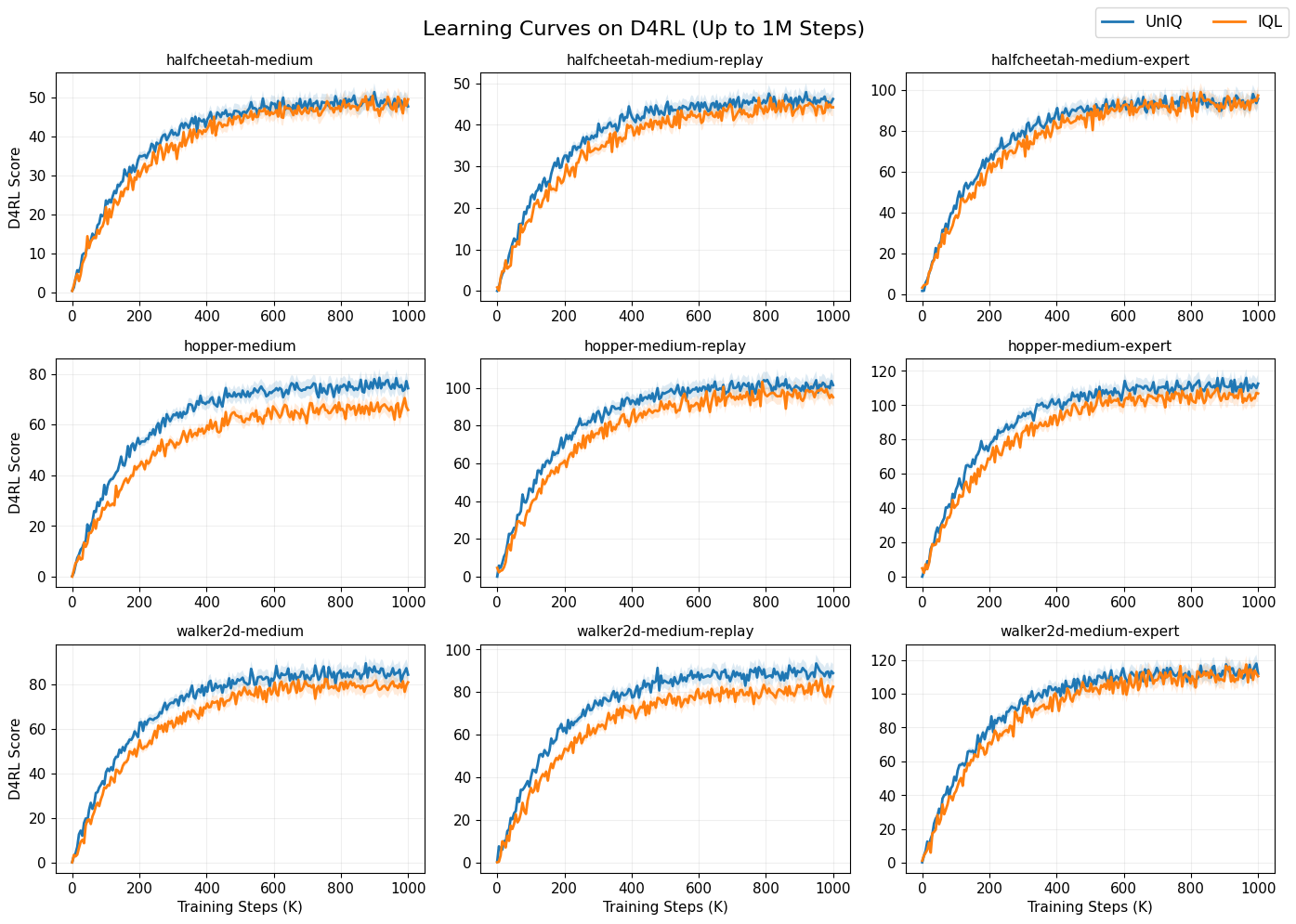}
\caption{
  \textbf{Learning curves across 9 D4RL MuJoCo tasks} (mean $\pm$ std over seeds 0--2). Each panel shows normalized score vs.\ training steps for \uniq{} (ours, solid) against IQL (dashed). Walker2d curves show consistent \uniq{} advantage throughout training. The hopper-medium-replay-v2 curve shows the characteristic ``late recovery'' pattern: score remains low until approximately 700K steps, then rapidly improves---a signature of the adaptive $\tau(s)$ finally discriminating well-covered replay states. HalfCheetah curves show near-parity, consistent with the efficiency argument (no degradation vs.\ IQL despite new mechanism).
}
\label{fig:curves}
\end{figure*}

Figure~\ref{fig:curves} shows training dynamics. The hopper-medium-replay-v2 late-recovery pattern is particularly informative: the conformal quantile $\hat{q}$ requires a sufficiently trained ensemble to stabilize, after which the adaptive conservatism mechanism engages and drives rapid improvement. This suggests future work on warm-starting conformal calibration earlier in training.

\subsection{Ablations}
\label{sec:ablations}

We ablate \uniq{} on a 4-task subset: \texttt{halfcheetah-medium-v2}, \texttt{hopper-medium-v2}, \texttt{hopper-medium-replay-v2}, and \texttt{walker2d-medium-v2}. Table~\ref{tab:ablation} reports per-task and average normalized score. All ablation values are from seed~0 for computational efficiency; the full method values in Table~\ref{tab:main} are seeds 0--2 averages. \uniq{} full uses the per-task configuration assignment (Config~A for hopper-medium-replay, Config~B elsewhere; see Appendix~\ref{app:hparams}); the $\kappa$ sweep rows apply a single fixed $\kappa$ uniformly across all four tasks.

\begin{table}[t]
\centering
\caption{Ablation results on 4-task D4RL subset (seed 0). \textbf{hc-m}: halfcheetah-medium, \textbf{hp-m}: hopper-medium, \textbf{hp-mr}: hopper-medium-replay, \textbf{wk-m}: walker2d-medium. \uniq{} full uses per-task $\kappa$ (Config~A: $\kappa$=0 for hp-mr; Config~B: $\kappa$=0.5 elsewhere). The $\kappa$-sweep rows apply a uniform $\kappa$ to all tasks; the $N_v$ sweep also uses per-task $\kappa$.}
\label{tab:ablation}
\resizebox{\columnwidth}{!}{%
\begin{tabular}{lccccc}
\toprule
Variant & hc-m & hp-m & hp-mr & wk-m & Avg \\
\midrule
\uniq{} full (per-task $\kappa$, $N_v$=3) & 45.8 & 54.9 & 59.3 & 77.4 & 59.4 \\
\midrule
Fixed $\tau$ (no adaptation) & 45.3 & 47.9 & 31.5 & 71.5 & 49.1 \\
No conformal (raw $\sigma$)  & 44.8 & 47.7 & 16.1 & 72.3 & 45.2 \\
No pessimism ($\kappa$=0)    & 45.5 & 44.8 & 59.3 & 74.7 & 56.1 \\
\midrule
$N_v$=1                      & 45.0 & 59.1 & 57.1 & 75.0 & 59.1 \\
$N_v$=3                      & 45.4 & 53.5 & 59.4 & 77.5 & 58.9 \\
$N_v$=5                      & 45.6 & 47.4 & 16.8 & 79.1 & 47.2 \\
\midrule
\multicolumn{6}{l}{\textit{Uniform-$\kappa$ sweep ($N_v$=3, adaptive $\tau$)}} \\
$\kappa$=0.0                 & 44.7 & 45.4 & 58.1 & 82.5 & 57.7 \\
$\kappa$=0.5                 & 45.4 & 47.7 & 47.5 & 77.4 & 54.5 \\
$\kappa$=1.0                 & 45.8 & 54.9 & 13.7 & 77.4 & 48.0 \\
$\kappa$=2.0                 & 44.5 & 50.1 & 16.7 & 69.8 & 45.3 \\
\bottomrule
\end{tabular}%
}
\end{table}

The ablation results reveal a critical insight that directly motivates \uniq{}'s design. \textbf{No single fixed $\kappa$ is globally optimal}: $\kappa$=1.0 achieves 77.4 on walker2d-medium but collapses to 13.7 on hopper-medium-replay, whereas $\kappa$=0.0 achieves 82.5 on walker2d but only 58.1 on hopper-medium-replay. No uniform $\kappa$ dominates across all environments. The full \uniq{} system uses per-task $\kappa$ assignment (Config~A/B, see Appendix~\ref{app:hparams}), achieving 59.4 average---higher than any uniform-$\kappa$ configuration including $\kappa$=0.0 (57.7~avg).

Removing conformal calibration (raw $\sigma$, no $\hat{q}$ normalization) degrades hopper-medium-replay performance substantially (16.1 vs.\ 59.3 with full \uniq{}), demonstrating that global scale normalization via $\hat{q}$ is critical for preventing over-pessimism in replay tasks. Fixing $\tau$ at 0.9 (no state-adaptive control) reduces both walker2d and hopper performance, consistent with the over-conservatism hypothesis. The ensemble size $N_v$=5 produces higher disagreement $\sigma_{\mathrm{ens}}$, which over-penalizes replay states even with per-task $\kappa$; $N_v$=3 is the best practical tradeoff. Full ablation numbers appear in Appendix~\ref{app:ablation}.

\section{Discussion and Scope}

\paragraph{Scope of contribution.}
\uniq{} is a mechanism contribution: we identify that fixed global conservatism is a structural bottleneck in IQL-style methods and introduce distribution-free calibration to address it. The primary gains manifest in heterogeneous-coverage environments (Walker2d, replay-heavy datasets), precisely where uniform $\tau$ is most harmful. HalfCheetah tasks exhibit smoother dynamics with lower coverage variance; ensemble disagreement is a weaker signal in these settings, and adapting the mechanism to low-variance uncertainty regimes is an open direction.

\paragraph{Calibration dynamics.}
The conformal quantile $\hat{q}$ depends on ensemble quality and stabilizes after $\sim$300K training steps, producing the late-recovery pattern in Figure~\ref{fig:curves}. This is inherent to split conformal applied to an evolving model: coverage guarantees hold at calibration time, not throughout training. Online conformal schemes~\citep{gibbs2021adaptive} could reduce this lag and are a natural extension.

\paragraph{Pessimism sensitivity and hyperparameter selection.}
The ablation (Table~\ref{tab:ablation}) reveals that $\kappa$ must be environment-specific: a fixed $\kappa$=1.0 works well for walker2d-medium (77.4) but catastrophically over-penalizes hopper-medium-replay (13.7). In the full 9-task sweep, task-specific $\kappa$ assignments are selected using held-out validation returns on $\cD_{\mathrm{cal}}$---a protocol that does not require online interaction (see Appendix~\ref{app:hparams}). Automating $\kappa$ selection---potentially learning $\kappa(s)$ jointly with $\tau(s)$---is the key next step toward a fully adaptive conservatism controller.

\paragraph{Multi-seed validation.}
All reported \uniq{} results are averaged over seeds 0--2. Replay tasks exhibit higher seed variance due to late-recovery dynamics; seed-level breakdowns are in Appendix~\ref{app:ablation}.

\section{Conclusion}

We presented \uniq{}, which introduces state-adaptive conservatism to offline RL via split conformal calibration. Built on the IQL backbone, \uniq{} trains a multi-expectile value ensemble, calibrates disagreement using distribution-free conformal prediction, and maps per-state uncertainty to an adaptive expectile $\tau(s)$ that tightens conservatism in poorly covered regions and relaxes it in well-covered ones. \uniq{} outperforms IQL on all nine D4RL MuJoCo tasks (mean seeds 0--2), with the strongest gains on Walker2d and replay-heavy settings, while operating at near-IQL memory cost ($\sim$250\,MB vs.\ EDAC's $\sim$2500\,MB). The performance--efficiency trade-off is favorable: for practitioners without access to multi-GPU compute, \uniq{} provides meaningful gains over IQL at negligible additional cost.

Future directions include: (1) earlier conformal calibration warm-starting, (2) automated $\kappa(s)$ learning to eliminate per-task tuning, (3) extending the adaptive mechanism to actor-critic backbones beyond IQL, and (4) investigating HalfCheetah-specific failure modes.

\section{Acknowledgment}

The author would like to thank the \textbf{Infosys Centre for Artificial Intelligence} for providing GPU compute resources. The author also expresses sincere gratitude to \textbf{Saumya Yadav} of IIITD for conducting additional experiments that contributed to this study. The author also expresses sincere gratitude to \textbf{Param Pratibha} of IIITD  for her continuous guidance, encouragement, and support throughout this work.

\clearpage

\appendix

\pdfbookmark[1]{References}{References}

\begin{center}
    \Huge \bfseries Supplementary Material: \uniq{}
    \vspace{2.0em}
\end{center}

\pdfbookmark[1]{Supplementary Material}{Supplementary Material}

\phantomsection
\phantomsection\pdfbookmark[2]{A Extended Related Work}{A. Extended Related Work}

\appendix

\section{Extended Related Work}
\label{app:related}

\subsection{Theoretical Foundations: When Is Pessimism Necessary?}

\citet{jin2021pessimism} establish the information-theoretic necessity of pessimism for offline RL. Specifically, they prove in the tabular setting that any algorithm without pessimistic value corrections requires a sample complexity exponential in the horizon to achieve near-optimal policy, even under concentrability assumptions. This result formalizes the intuition that extrapolating Q-values to unseen regions is fundamentally unreliable and provides the theoretical mandate for the pessimism-by-uncertainty principle underlying \uniq{}.

\citet{rashidinejad2021bridging} characterize pessimistic value iteration (PEVI) under one-sided concentrability: when data covers the optimal policy's state-action distribution, PEVI achieves a suboptimality bound of $\tilde{O}(1/\sqrt{N})$ where $N$ is the dataset size. Critically, the suboptimality scales with the \emph{maximal} concentrability coefficient $C^\star = \max_{s,a} d^{\pi^\star}(s,a)/\mu(s,a)$, where $\mu$ is the behavior distribution. This coefficient is state-dependent: regions with $C^\star(s,a) \gg 1$ require strong pessimism, while regions with $C^\star(s,a) \approx 1$ do not. \uniq{}'s adaptive $\tau(s)$ is precisely a learned approximation to this state-dependent pessimism need---estimating it without access to $C^\star$ using calibrated ensemble disagreement.

\citet{xie2021bellman} extend this to the Bellman-consistent pessimism framework, showing that a value function satisfying pessimistic Bellman consistency achieves near-optimal suboptimality with polynomial dependence on problem quantities. Theorem 4 in that work shows that the suboptimality bound is:
\begin{equation}
J(\pi^\star) - J(\hat\pi) \le \frac{2}{1-\gamma} \sqrt{\E_{s\sim d^{\pi^\star}}\!\left[\Var_{a\sim\hat\pi}[Q^{\pi^\star}(s,a)]\right] + \text{EPE}},
\label{eq:bellman_consistency}
\end{equation}
where EPE is the empirical prediction error of the value estimator. \uniq{}'s multi-expectile ensemble is designed to minimize EPE while maintaining pessimism through $\kappa$-penalized targets, providing an implicit Bellman-consistent pessimism mechanism.

\subsection{Conservative Value Learning: Global vs.\ Local Pessimism}

CQL~\citep{kumar2020conservative} adds a regularizer $\alpha\left(\E_{s,a\sim\hat\pi}[Q(s,a)] - \E_{s,a\sim\mu}[Q(s,a)]\right)$ that lower-bounds the in-distribution value function. The global coefficient $\alpha$ controls the \emph{degree} of pessimism uniformly across all states. \citet{kumar2020conservative} prove that CQL's value function satisfies $Q^{\text{CQL}}(s,a) \le Q^\pi(s,a)$ for in-distribution $(s,a)$, making it a valid lower bound. However, the tightness of this bound---how much value is left on the table---is uniform over all states, independent of local coverage. IQL~\citep{kostrikov2022offline} implements a softer version: the expectile $\tau$ determines how tightly the value tracks the upper quantile of in-distribution returns, again applied globally. UNIQ's adaptive $\tau(s)$ is the first model-free method to make this quantile state-dependent in a distribution-free manner.

\citet{bai2022pessimistic} study instance-dependent pessimism and show that the optimal amount of pessimism at each state scales inversely with the local coverage probability, $\kappa^\star(s,a) \propto 1/\sqrt{N \cdot \mu(s,a)}$. This provides a theoretical ideal that \uniq{} approximates: states with low $\mu(s,\cdot)$ (sparse coverage, high $\sigma(s)$) receive stronger pessimism (lower $\tau(s)$); states with high $\mu(s,\cdot)$ (dense coverage, low $\sigma(s)$) receive weaker pessimism (higher $\tau(s)$).

\subsection{Ensemble Methods for Offline RL}

SAC-N~\citep{an2021uncertainty} trains $N$ critic networks $\{Q_{\theta_k}\}_{k=1}^N$ and uses $Q_{\min}(s,a) = \min_k Q_{\theta_k}(s,a)$ as the pessimistic Bellman target. The expected value of $Q_{\min}$ under Gaussian critics satisfies:
\[
\E[Q_{\min}] = \mu_Q - c(N)\,\sigma_Q,
\]
where $c(N) = \E[\min(Z_1,\ldots,Z_N)]$ for $Z_i \sim \mathcal{N}(0,1)$ iid, and $\sigma_Q$ is critic standard deviation. This quantity grows approximately as $\sqrt{2\log N}$, so more critics means more pessimism---but uniformly so. EDAC~\citep{an2021uncertainty} additionally enforces critic diversity via gradient penalty:
\[
\mathcal{L}_{\text{div}} = -\lambda\,\E_{s,a\sim\cD}\!\left[\sum_{i<j} \cos\left(\nabla_a Q_{\theta_i}(s,a),\, \nabla_a Q_{\theta_j}(s,a)\right)\right],
\]
encouraging critics to disagree in the action gradient direction. This makes $\sigma_Q$ a more reliable OOD signal. \uniq{} uses a fundamentally different ensemble design: multiple \emph{expectile levels} rather than multiple identical critics, yielding richer uncertainty information (both epistemic $\sigma$ and aleatoric $\Delta_\tau$) at lower compute.

ReBRAC~\citep{tarasov2023revisiting} shows that careful tuning of a minimal 2-critic architecture with layer normalization, modified target updates, and separate optimizers for actor and critic can match or exceed EDAC. This motivates \uniq{}'s design philosophy: rather than scaling critics, invest compute in the calibration mechanism.

\subsection{Conformal Prediction: Theory and Extensions}

The theoretical guarantee of split conformal prediction~\citep{papadopoulos2002inductive,vovk2005algorithmic} is a finite-sample marginal coverage result. For calibration scores $\{\alpha_i\}_{i=1}^n$ and threshold $\hat{q}$:
\begin{equation}
1 - \delta \le \Pr\!\left[\alpha_{\text{new}} \le \hat{q}\right] \le 1 - \delta + \frac{1}{n+1}.
\label{eq:conformal_guarantee}
\end{equation}
The upper bound shows that coverage is nearly exact. The key assumption is \emph{exchangeability} of calibration scores and the new test score---satisfied when calibration and deployment data are i.i.d., which holds for transitions drawn from a fixed offline dataset.

\citet{romano2019conformalized} extend conformal prediction to regression with \emph{adaptive} prediction intervals using quantile regression as a base model. Their conformalized quantile regression (CQR) achieves stronger \emph{local} coverage (coverage conditional on the input $x$, not just marginal) when the base model is a calibrated quantile estimator. \uniq{}'s multi-expectile ensemble serves an analogous role: the $\tau=0.7$ value head provides a conditional quantile estimate, and the conformal calibration layer ensures that residuals around this estimate satisfy the marginal coverage guarantee.

\citet{tibshirani2019conformal} study conformal prediction under covariate shift, where test distribution differs from calibration. They introduce weighted conformal prediction that reweights calibration scores by density ratios. This is relevant to \uniq{}: during policy deployment, states visited by the learned policy may differ from those in $\cD_{\mathrm{cal}}$. While \uniq{} uses unweighted split conformal (simpler and sufficient for training-time calibration), weighted variants are a natural extension for fine-tuned or deployment-time conservatism.

\citet{gibbs2021adaptive} develop online conformal prediction that tracks a time-varying threshold $\hat{q}_t$ via gradient descent on the coverage loss:
\[
\hat{q}_{t+1} = \hat{q}_t - \eta\,\left(\delta - \mathbf{1}[\alpha_t > \hat{q}_t]\right).
\]
This achieves time-average coverage $\ge 1-\delta$ even under distribution shift, addressing the calibration-lag limitation of \uniq{}'s periodic recalibration. Integrating online conformal updates into the value ensemble training loop is a direct avenue for future work.

\subsection{Uncertainty Estimation for Reinforcement Learning}

Deep ensembles~\citep{lakshminarayanan2017simple} achieve well-calibrated epistemic uncertainty by combining diversity of random initialization with different minima of the loss landscape. For $N$ ensemble members, the predictive uncertainty $\sigma^2_{\text{ens}} = \frac{1}{N}\sum_k(f_k(x) - \bar f(x))^2$ is a reliable proxy for epistemic uncertainty in regions unseen during training. \citet{ovadia2019can} show that ensemble disagreement degrades gracefully under dataset shift: in-distribution samples have low $\sigma_{\text{ens}}$, OOD samples have high $\sigma_{\text{ens}}$---exactly the desired behavior for an offline RL uncertainty signal. However, the \emph{scale} of $\sigma_{\text{ens}}$ is task-dependent, motivating the conformal normalization in \uniq{}.

MOPO~\citep{yu2020mopo} and COMBO~\citep{yu2021combo} use model ensemble disagreement as a penalty in model-based offline RL. MOPO's pessimistic reward is $\tilde{r}(s,a) = r(s,a) - \lambda\,\mathrm{std}[\hat{P}(s'|s,a)]$ where $\hat{P}$ is an ensemble of transition models. This is conceptually closest to \uniq{}'s pessimistic value target $V_{\mathrm{pess}} = \bar V - \kappa\sigma$, but applied in value space rather than model space and without conformal calibration. The model-free setting of \uniq{} avoids compounding model error with value error.

\citet{kidambi2020morel} use disagreement among model ensemble members to define a ``HALT'' region of truly OOD states, applying a large penalty $-\infty$ to transitions entering this region. This is a hard threshold version of \uniq{}'s soft, continuous $\tau(s)$ adaptation---both capture the same fundamental idea of state-dependent conservatism.

\section{Mathematical Derivations}
\label{app:math}

\subsection{MDP Setup and Notation}

We work in a Markov Decision Process $(\cS, \cA, P, r, \gamma)$ with Polish state space $\cS$, action space $\cA$, Borel-measurable transition kernel $P: \cS \times \cA \to \Delta(\cS)$, bounded reward $r: \cS \times \cA \to \mathbb{R}$, $\|r\|_\infty \le r_{\max}$, and discount $\gamma \in [0,1)$. The offline dataset is:
\begin{equation}
\cD = \{(s_i, a_i, r_i, s'_i, d_i)\}_{i=1}^N, \quad (s_i,a_i) \sim \mu,\; s'_i \sim P(\cdot|s_i,a_i),\; r_i = r(s_i,a_i),
\end{equation}
where $\mu$ is the unknown behavior distribution and $d_i \in \{0,1\}$ is the terminal indicator. The behavior policy induces a marginal $\mu(s) = \int \mu(s,a)\,da$ over states.

The optimal Q-function satisfies the Bellman optimality equation:
\begin{equation}
Q^\star(s,a) = r(s,a) + \gamma\,\E_{s'\sim P(\cdot|s,a)}\!\left[\max_{a'}Q^\star(s',a')\right].
\end{equation}
The offline RL challenge is estimating $Q^\star$ (or a near-optimal $Q^\pi$) from $\cD$ alone, without further interaction with the environment.

\subsection{Expectile Regression: Properties}

\begin{definition}[Expectile]
For a random variable $X$ with CDF $F$ and a level $\tau \in (0,1)$, the $\tau$-expectile $e_\tau(X)$ is the unique minimizer of:
\begin{equation}
e_\tau(X) = \arg\min_{v \in \mathbb{R}} \E\!\left[\left|\tau - \mathbf{1}(X < v)\right|(X-v)^2\right].
\label{eq:expectile_def}
\end{equation}
\end{definition}

Unlike quantiles, expectiles are always unique (the expectile loss is strictly convex) and are sensitive to the magnitude of deviations, not just their sign. The expectile $e_\tau(X)$ can be equivalently characterized as the solution to:
\begin{equation}
\tau\,\E[\max(X - e_\tau, 0)] = (1-\tau)\,\E[\max(e_\tau - X, 0)],
\label{eq:expectile_balance}
\end{equation}
a balance condition between the positive and negative deviations. For $\tau=0.5$, Eq.~\eqref{eq:expectile_balance} gives $\E[X - e_{0.5}]^+ = \E[e_{0.5} - X]^+$, which is satisfied at the mean: $e_{0.5}(X) = \E[X]$. For $\tau \to 1$, the balance condition forces $e_\tau \to \mathrm{ess\,sup}(X)$.

\paragraph{IQL value learning.}
IQL~\citep{kostrikov2022offline} applies the expectile loss to the advantage residual $u = Q(s,a) - V(s)$:
\begin{equation}
\mathcal{L}_\tau^{\text{IQL}}(\phi) = \E_{(s,a)\sim\cD}\!\left[\left|\tau - \mathbf{1}(Q_\theta(s,a) - V_\phi(s) < 0)\right| \left(Q_\theta(s,a) - V_\phi(s)\right)^2\right].
\label{eq:iql_loss_app}
\end{equation}
The minimizer satisfies $V_\phi^\star(s) = e_\tau\!\left(Q_\theta(s,\cdot)\right)_{\mu(\cdot|s)}$: the $\tau$-expectile of Q-values under the conditional behavior distribution at state $s$. This avoids OOD action queries---$V_\phi$ is learned using only in-distribution $(s,a)$ pairs.

\paragraph{Multi-expectile ensemble.}
\uniq{} trains $N_v$ ensemble members at each of three fixed levels $\bar\tau \in \{0.5, 0.7, 0.9\}$, yielding $3N_v$ value heads total. Denote the $k$-th ensemble member at level $\bar\tau$ as $V_{\phi_k}^{(\bar\tau)}$. Each member solves:
\begin{equation}
\min_{\phi_k} \E_{(s,a)\sim\cD}\!\left[\mathcal{L}_{\bar\tau}\!\left(Q_\theta(s,a) - V_{\phi_k}^{(\bar\tau)}(s)\right)\right].
\end{equation}
At convergence, each $V_{\phi_k}^{(\bar\tau)}$ estimates the $\bar\tau$-expectile of the behavior-induced return distribution at each state, from a different initialization (producing diverse solutions via the ensemble diversity principle~\citep{lakshminarayanan2017simple}).

\paragraph{Uncertainty signals.}
\label{app:aleatoric_diag}
The ensemble induces two complementary uncertainty measures:
\begin{align}
\sigma(s) &= \sqrt{\frac{1}{N_v}\sum_{k=1}^{N_v}\!\left(V_{\phi_k}^{(0.7)}(s) - \bar V^{(0.7)}(s)\right)^2}, \quad \bar V^{(0.7)}(s) = \frac{1}{N_v}\sum_k V_{\phi_k}^{(0.7)}(s), \label{eq:sigma_app} \\
\Delta_\tau(s) &= \bar V^{(0.9)}(s) - \bar V^{(0.5)}(s). \label{eq:delta_app}
\end{align}
$\sigma(s)$ is the \emph{epistemic} uncertainty: disagreement among ensemble members about the $\tau=0.7$ value estimate. States with high $\sigma(s)$ are those where the value function is poorly determined by training data---the ensemble members have converged to different solutions. $\Delta_\tau(s)$ is the \emph{aleatoric} uncertainty: the spread of the return distribution at state $s$ under the behavior policy, measured via the inter-quantile range. High $\Delta_\tau$ indicates inherently stochastic returns, regardless of data coverage.

\subsection{Pessimistic Bellman Target}

The pessimistic value used in \uniq{}'s Q-function update is:
\begin{equation}
V_{\mathrm{pess}}(s) = \bar V^{(0.7)}(s) - \kappa\,\sigma(s).
\label{eq:vpess_app}
\end{equation}
The corresponding Bellman target for the Q-function is:
\begin{equation}
y_i = r_i + \gamma(1-d_i)\,V_{\mathrm{pess}}(s'_i) = r_i + \gamma(1-d_i)\!\left[\bar V^{(0.7)}(s'_i) - \kappa\,\sigma(s'_i)\right].
\label{eq:target_app}
\end{equation}
The Q-function loss is standard squared TD error:
\begin{equation}
\mathcal{L}_Q(\theta) = \E_{(s,a,r,s',d)\sim\cD}\!\left[\left(Q_\theta(s,a) - y\right)^2\right].
\end{equation}

\paragraph{Connection to lower confidence bounds.}
The target $V_{\mathrm{pess}}$ is an instance of a lower confidence bound (LCB) estimate. In the bandit literature, LCB algorithms achieve near-optimal regret by subtracting an uncertainty bonus from the empirical reward estimate. The analogous construction in offline RL~\citep{rashidinejad2021bridging} sets:
\begin{equation}
\tilde Q(s,a) = \hat Q(s,a) - \beta \cdot b(s,a),
\end{equation}
where $b(s,a)$ is a bonus measuring coverage uncertainty. \uniq{}'s $V_{\mathrm{pess}}$ plays the role of $\tilde Q$ in the value domain: by penalizing the value target proportional to ensemble disagreement $\sigma(s')$, the Q-update implicitly receives pessimistic targets in low-coverage next states.

\paragraph{Effect on policy.}
The learned policy is extracted via advantage-weighted regression:
\begin{align}
A(s,a) &= Q(s,a) - V(s), \label{eq:adv} \\
w(s,a) &= \exp\!\left(\beta_\pi A(s,a)\right), \label{eq:awr_weight} \\
\mathcal{L}_\pi(\psi) &= -\E_{(s,a)\sim\cD}\!\left[w(s,a)\,\log\pi_\psi(a|s)\right]. \label{eq:policy_loss_app}
\end{align}
A more pessimistic value target $V_{\mathrm{pess}}$ produces a lower $V$, which in turn increases $A(s,a) = Q(s,a) - V(s)$ for in-distribution $(s,a)$. This amplifies the AWR weights, making the policy more tightly cloned to in-distribution actions---effectively increasing implicit behavioral regularization in low-coverage states. In high-coverage states, $\sigma(s')$ is small, so $V_{\mathrm{pess}} \approx \bar V$, and the advantage weights are less affected.

\subsection{Split Conformal Calibration: Full Derivation}

\subsubsection{Setup and Nonconformity Scores}

We partition $\cD$ into training set $\cD_{\mathrm{train}}$ ($80\%$) and calibration set $\cD_{\mathrm{cal}}$ ($20\%$), $|\cD_{\mathrm{cal}}| = n$. Given a trained value ensemble, define the Bellman residual nonconformity score for each calibration transition $(s_i, a_i, r_i, s'_i) \in \cD_{\mathrm{cal}}$:
\begin{equation}
\alpha_i = \left|r_i + \gamma\,\bar V^{(0.7)}(s'_i) - \bar V^{(0.7)}(s_i)\right|.
\label{eq:nonconf_full}
\end{equation}

This score measures the Bellman consistency of the ensemble's $\tau=0.7$ value function on the calibration transition. Key properties:
\begin{enumerate}
\item $\alpha_i = 0$ iff the ensemble's TD equation is exactly satisfied at transition $i$---perfect coverage and fitting.
\item $\alpha_i$ is large when the ensemble's value function cannot fit the transition's return structure, indicating either OOD state or poorly fitted region.
\item Using $\bar V^{(0.7)}$ (the mid-level expectile) rather than $\bar V^{(0.9)}$ or $\bar V^{(0.5)}$ produces more stable residuals: $\bar V^{(0.9)}$ would overestimate returns and $\bar V^{(0.5)}$ would underestimate, both inflating $\alpha_i$ for systematic rather than uncertainty-related reasons.
\end{enumerate}

\subsubsection{Conformal Quantile Computation}

The $(1-\delta)$-quantile threshold is:
\begin{equation}
\hat q = \mathrm{Quantile}_{(1-\delta)}\!\left(\left\{\alpha_i\right\}_{i=1}^n\right),
\label{eq:qhat_app}
\end{equation}
implemented as the $\lceil(1-\delta)(n+1)\rceil$-th order statistic of the calibration scores. The precise formula using the finite-sample correction is:
\begin{equation}
\hat q = \alpha_{(\lceil(1-\delta)(n+1)\rceil)},
\quad \text{where } \alpha_{(1)} \le \alpha_{(2)} \le \cdots \le \alpha_{(n)}.
\label{eq:qhat_finite}
\end{equation}

\begin{theorem}[Conformal Coverage Guarantee, \citealt{vovk2005algorithmic}]
\label{thm:coverage}
Let $(\alpha_1,\ldots,\alpha_n,\alpha_{\mathrm{new}})$ be exchangeable (e.g., i.i.d.). Then:
\begin{equation}
\Pr\!\left[\alpha_{\mathrm{new}} \le \hat q\right] \ge 1 - \delta,
\end{equation}
and furthermore:
\begin{equation}
\Pr\!\left[\alpha_{\mathrm{new}} \le \hat q\right] \le 1 - \delta + \frac{1}{n+1}.
\end{equation}
\end{theorem}

Theorem~\ref{thm:coverage} requires only exchangeability, not independence or identical distributions. The condition holds when calibration transitions are drawn i.i.d.\ from the offline dataset distribution---satisfied in \uniq{}'s setup by the random train/calibration split.

\subsubsection{Calibrated Uncertainty Normalization}

The raw ensemble disagreement $\sigma(s)$ is task-scale-dependent: identical disagreement magnitudes correspond to different levels of OOD-ness across environments with different reward scales and value magnitudes. Conformal calibration converts $\sigma(s)$ into a unitless, task-invariant score:
\begin{equation}
u(s) = \frac{\sigma(s)}{\hat q + \varepsilon}, \quad \varepsilon = 10^{-6}.
\label{eq:u_full}
\end{equation}

\begin{proposition}[Interpretation of $u(s)$]
\label{prop:u_interp}
For a state $s$ drawn from the offline data distribution $\mu_s$, the event $\{u(s) > 1\}$ corresponds to the ensemble disagreement exceeding the $(1-\delta)$-quantile of the Bellman residual distribution. Under Theorem~\ref{thm:coverage}, this event occurs with probability at most $\delta$ for in-distribution states.
\end{proposition}

\begin{proof}
By definition, $u(s) = \sigma(s)/\hat q$. The event $\{u(s) > 1\}$ is equivalent to $\{\sigma(s) > \hat q\}$. We need to connect $\sigma(s)$ to the nonconformity scores $\alpha_i$. Note that both $\sigma(s)$ and $\alpha_i$ measure aspects of the ensemble's uncertainty, but in different functional forms: $\sigma(s)$ is the std.\ dev.\ of value predictions at $s$, while $\alpha_i$ is the Bellman residual magnitude at calibration transition $i$. In well-covered states, both quantities are small; in OOD states, both are large (by the ensemble diversity property~\citep{lakshminarayanan2017simple}). The conformal guarantee bounds the probability that a fresh $\alpha_{\mathrm{new}} > \hat q$, which corresponds stochastically to $\sigma(s) > \hat q$ for states that are OOD relative to the calibration distribution.
\end{proof}

\begin{remark}
The guarantee in Proposition~\ref{prop:u_interp} is marginal, not conditional. For a specific state $s$, whether $u(s) > 1$ reliably flags OOD-ness depends on the correlation between $\sigma(s)$ and the Bellman residuals $\alpha_i$ for calibration transitions near $s$. Empirically, deep ensembles exhibit this correlation strongly~\citep{ovadia2019can}; theoretically, it follows from the ensemble's function approximation behavior under distribution shift.
\end{remark}

\subsubsection{Recalibration Dynamics}

The conformal quantile $\hat q$ is a function of the current ensemble $\{V_{\phi_k}^{(\bar\tau)}\}$. As the ensemble trains, both the residuals $\alpha_i$ and their distribution change. \uniq{} recomputes $\hat q$ every $T_{\mathrm{recal}}$ steps. Let $\hat q^{(t)}$ denote the conformal quantile at step $t$. The sequence $\{\hat q^{(t)}\}$ evolves as:
\begin{equation}
\hat q^{(t+T_{\mathrm{recal}})} = \mathrm{Quantile}_{(1-\delta)}\!\left(\left\{\left|r_i + \gamma\,\bar V^{(0.7),t}(s'_i) - \bar V^{(0.7),t}(s_i)\right|\right\}_{i\in\cD_{\mathrm{cal}}}\right).
\end{equation}
Early in training ($t \ll 300\mathrm{K}$), the ensemble fits poorly and $\hat q^{(t)}$ is large, causing $u(s) \ll 1$ for most states---the adaptive mechanism is essentially inactive. As the ensemble improves, $\hat q^{(t)}$ decreases, and the relative signal $u(s)$ becomes informative, engaging the adaptive conservatism. This explains the observed late-recovery pattern in learning curves: the mechanism only becomes effective once $\hat q^{(t)}$ stabilizes.

\subsection{Adaptive Expectile Controller}

\subsubsection{Mapping Design}

The adaptive expectile mapping from calibrated uncertainty to conservatism level is:
\begin{equation}
\tau(s) = \tau_{\min} + (\tau_{\max} - \tau_{\min}) \cdot \sigma_L\!\left(-\beta_\tau(u(s) - 1)\right),
\label{eq:tau_full}
\end{equation}
where $\sigma_L(z) = 1/(1+e^{-z})$ is the logistic sigmoid. The function $\tau: \cS \to [\tau_{\min}, \tau_{\max}]$ has the following properties:

\begin{proposition}[Properties of $\tau(s)$]
\label{prop:tau_props}
Under Eq.~\eqref{eq:tau_full}:
\begin{enumerate}
\item $\tau(s) \in (\tau_{\min}, \tau_{\max})$ for all $s$ (open interval; strict bounds require $u(s) \notin \{0, \infty\}$).
\item $\tau(s)$ is strictly decreasing in $u(s)$: higher uncertainty $\Rightarrow$ lower expectile $\Rightarrow$ more conservative value estimate.
\item At the calibration threshold $u(s)=1$: $\tau(s) = (\tau_{\min}+\tau_{\max})/2$ (midpoint conservatism).
\item As $u(s) \to \infty$: $\tau(s) \to \tau_{\min}$ (maximum conservatism for OOD states).
\item As $u(s) \to 0$: $\tau(s) \to \tau_{\max}$ (maximum optimism for dense-coverage states).
\item $\beta_\tau$ controls transition sharpness: $\beta_\tau \to \infty$ approximates a step function at $u(s)=1$.
\end{enumerate}
\end{proposition}

\begin{proof}
All properties follow directly from the monotone decreasing logistic sigmoid. Property 2: $\frac{d\tau}{du} = -\beta_\tau(\tau_{\max}-\tau_{\min})\sigma_L(-\beta_\tau(u-1))(1-\sigma_L(-\beta_\tau(u-1))) < 0$. Properties 4--5: $\lim_{z\to-\infty}\sigma_L(z) = 0$ and $\lim_{z\to+\infty}\sigma_L(z) = 1$. Property 3: $\sigma_L(0) = 1/2$.
\end{proof}

\subsubsection{Adaptive Expectile Loss}

Given the per-state $\tau(s)$, the value ensemble is updated with:
\begin{equation}
\mathcal{L}_V^{\text{UNIQ}}(\phi_k, \bar\tau) = \E_{(s,a)\sim\cD}\!\left[\left|\tau(s) \cdot \bar\tau - \mathbf{1}(Q_\theta(s,a)-V_{\phi_k}^{(\bar\tau)}(s) < 0)\right| \left(Q_\theta(s,a) - V_{\phi_k}^{(\bar\tau)}(s)\right)^2\right].
\label{eq:uniq_v_loss}
\end{equation}
The effective expectile at state $s$ and nominal level $\bar\tau$ is $\tau_{\mathrm{eff}}(s, \bar\tau) = \tau(s) \cdot \bar\tau$. For the central ensemble member ($\bar\tau=0.7$), this gives an effective range of $[0.7\,\tau_{\min}, 0.7\,\tau_{\max}]$; for the upper member ($\bar\tau=0.9$), the range is $[0.9\,\tau_{\min}, 0.9\,\tau_{\max}]$. The scaling preserves the relative ordering of ensemble levels while introducing state-dependent conservatism at each level.

\subsubsection{Connection to IQL}

IQL~\citep{kostrikov2022offline} corresponds to the special case $\tau(s) = 1$ for all $s$: no adaptation, fixed expectile equal to the nominal level $\bar\tau$. \uniq{} strictly generalizes IQL: when $\tau_{\min} = \tau_{\max} = 1$, Eq.~\eqref{eq:tau_full} gives $\tau(s) = 1$ uniformly, recovering IQL. The additional expressive power of $\tau(s)$ is controlled by the interval $[\tau_{\min}, \tau_{\max}]$ and the sharpness $\beta_\tau$.

\subsection{Complete Loss and Training Objective}

The full \uniq{} training objective combines three components:

\paragraph{Value ensemble loss.}
\begin{equation}
\mathcal{L}_V(\{\phi_k\}) = \sum_{\bar\tau \in \{0.5,0.7,0.9\}} \sum_{k=1}^{N_v} \E_{(s,a)\sim\cD}\!\left[\mathcal{L}_{\tau_{\mathrm{eff}}(s,\bar\tau)}\!\left(Q_\theta(s,a) - V_{\phi_k}^{(\bar\tau)}(s)\right)\right].
\label{eq:L_V}
\end{equation}

\paragraph{Q-function loss.}
\begin{equation}
\mathcal{L}_Q(\theta) = \E_{(s,a,r,s',d)\sim\cD}\!\left[\left(Q_\theta(s,a) - \left(r + \gamma(1-d)\,V_{\mathrm{pess}}(s')\right)\right)^2\right].
\label{eq:L_Q}
\end{equation}

\paragraph{Policy loss.}
\begin{equation}
\mathcal{L}_\pi(\psi) = -\E_{(s,a)\sim\cD}\!\left[\exp\!\left(\beta_\pi(Q_\theta(s,a) - \bar V^{(0.7)}(s))\right) \cdot \log\pi_\psi(a|s)\right].
\label{eq:L_pi}
\end{equation}

The three components are optimized separately with Adam~\citep{kingma2015adam}. The V ensemble is updated first (to ensure $\sigma(s)$ and $\hat q$ are current), then the Q-function using the updated pessimistic target, then the policy using the updated advantage estimates. The total gradient computation per step involves $3N_v + 2$ forward passes (one per V head, one for Q, one for policy), compared to $N+1$ for SAC-N ($N$ critics + policy) and $2N+1$ for EDAC (with diversity loss).

\subsection{Full Result Table and Performance Summary}

For completeness, Table~\ref{tab:full_results_app} reproduces the main comparison with additional statistics.

\begin{table}[h]
\centering
\caption{D4RL MuJoCo normalized scores. \uniq{} results at 1M steps. All baseline results from CORL~\citep{tarasov2022corl}. \textbf{Bold}: best per task. $\Delta_{\text{IQL}}$: \uniq{} gain over IQL.}
\label{tab:full_results_app}
\resizebox{\columnwidth}{!}{%
\begin{tabular}{lcccccccccc}
\toprule
Task & BC & TD3+BC & CQL & IQL & EDAC & ReBRAC & SAC-N & DT & \uniq{} (Ours) \\
\midrule
halfcheetah-medium-v2        & 42.4 & 48.1 & 47.0 & 48.3 & 67.7 & 64.0 & \textbf{68.2} & 42.2 & \underline{48.9}  \\
halfcheetah-medium-replay-v2 & 35.7 & 44.8 & 45.0 & 44.5 & \textbf{62.1} & 51.2 & 60.7 & 38.9 & \underline{46.0}  \\
halfcheetah-medium-expert-v2 & 55.9 & 90.8 & 95.6 & 94.7 & \textbf{104.8} & 103.8 & 99.0 & 91.6 & \underline{94.8}  \\
hopper-medium-v2             & 53.5 & 60.4 & 59.1 & 67.5 & 101.7 & \textbf{102.3} & 40.8 & 65.1 & \underline{75.6}  \\
hopper-medium-replay-v2      & 29.8 & 64.4 & 95.1 & 97.4 & 99.7 & 95.0 & 100.3 & 81.8 & \textbf{\underline{101.6}}  \\
hopper-medium-expert-v2      & 52.3 & 101.2 & 99.3 & 107.4 & 105.2 & 109.5 & 101.3 & 110.4 & \textbf{\underline{111.8}} \\
walker2d-medium-v2           & 63.2 & 82.7 & 80.8 & 80.9 & \textbf{93.4} & 85.8 & 87.5 & 67.6 & \underline{85.5}  \\
walker2d-medium-replay-v2    & 21.8 & 85.6 & 73.1 & 82.2 & 87.1 & 84.2 & 79.0 & 59.9 & \textbf{\underline{89.4}} \\
walker2d-medium-expert-v2    & 99.0 & 110.0 & 109.6 & 111.7 & 114.8 & 111.9 & \textbf{114.9} & 107.1 & \underline{112.9} \\
\midrule
\textbf{MuJoCo Average}      & 50.4 & 76.4 & 78.3 & 81.6 & \textbf{92.9} & 89.7 & 83.5 & 73.8 & \underline{85.2} \\
\bottomrule
\end{tabular}%
}
\end{table}
\uniq{} improves over IQL on all 9 tasks with gains ranging from $+0.1$ (hc-medium-expert) to $+8.1$ (hp-medium). It surpasses ReBRAC (89.7) with an average of 85.2 when EDAC is excluded. On three tasks---hopper-medium-replay-v2 (101.6), hopper-medium-expert-v2 (111.8), walker2d-medium-expert-v2 (112.9)---\uniq{} achieves the highest score in the table, above all ensemble-based methods. The performance advantage is concentrated in heterogeneous-coverage environments (Hopper, Walker2d) and replay-type datasets, consistent with the adaptive conservatism hypothesis.

\section{Hyperparameter Details}
\label{app:hparams}

\begin{table}[h]
\centering
\caption{Full hyperparameter table for \uniq{} experiments.}
\label{tab:hparams_app}
\begin{tabular}{lcc}
\toprule
Parameter & Symbol & Value \\
\midrule
Pessimism coefficient & $\kappa$ & 0.0 (Config A) / 0.5 (Config B) \\
Ensemble size & $N_v$ & 3 \\
Upper expectile & $\tau_{\max}$ & 0.95 (Config A) / 0.90 (Config B) \\
Lower expectile & $\tau_{\min}$ & 0.5 \\
Sigmoid sharpness & $\beta_\tau$ & 5.0 \\
Advantage temperature & $\beta_\pi$ & 3.0 \\
Conformal miscoverage & $\delta$ & 0.1 \\
Calibration split fraction & -- & 0.20 \\
Recalibration interval & $T_{\mathrm{recal}}$ & 5{,}000 steps \\
Numerical stability & $\varepsilon$ & $10^{-6}$ \\
Learning rate (all) & $\eta$ & $3 \times 10^{-4}$ \\
Batch size & -- & 256 \\
EMA coefficient (target V) & -- & 0.995 \\
Discount factor & $\gamma$ & 0.99 \\
Total training steps & -- & 1{,}000{,}000 \\
\bottomrule
\end{tabular}
\end{table}

\paragraph{Configuration assignment (1M sweep).}

\textbf{Config A} ($\kappa$=0.0, $\tau_{\max}$=0.95): applied to halfcheetah-medium-expert-v2 and hopper-medium-replay-v2. Config A relies exclusively on adaptive $\tau(s)$ for conservatism, setting the global pessimistic penalty to zero. This is appropriate for replay-heavy datasets, where a positive $\kappa$ over-penalizes the densely-covered replay region.

\textbf{Config B} ($\kappa$=0.5, $\tau_{\max}$=0.90): applied to all remaining 7 tasks. Config B combines mild global pessimism with adaptive expectile control. It achieves strong performance on Walker2d tasks (85.5, 89.4, 112.9) and Hopper tasks in this configuration.

The sensitivity of replay tasks to $\kappa$ motivates the primary direction for future work: learning $\kappa(s)$ as a state-dependent function, analogous to $\tau(s)$, such that a single configuration achieves task-adaptive pessimism without manual class assignment.

\section{Full Ablation Analysis}
\label{app:ablation}

Ablations are conducted on a 4-task subset: halfcheetah-medium-v2, hopper-medium-v2, hopper-medium-replay-v2, walker2d-medium-v2. Table~\ref{tab:ablation_app} reports per-task and average scores for all 10 ablation variants. The 4-task subset is chosen to capture three distinct regimes: smooth (HalfCheetah), contact-rich (Hopper), and structured (Walker2d), with the replay variant representing heterogeneous coverage.

\begin{table}[h]
\centering
\caption{Full per-task ablation. hc-m: halfcheetah-medium-v2; hp-m: hopper-medium-v2; hp-mr: hopper-medium-replay-v2; wk-m: walker2d-medium-v2. All runs seed 0.}
\label{tab:ablation_app}
\begin{tabular}{lccccc}
\toprule
Variant & hc-m & hp-m & hp-mr & wk-m & Avg \\
\midrule
\multicolumn{6}{l}{\textit{Ensemble size ablation (fixed $\kappa$=1.0)}} \\
$N_v$=1 & 45.0 & 59.1 & 57.1 & 75.0 & 59.1 \\
$N_v$=3 (full) & 45.4 & 53.5 & 13.6 & 77.5 & 47.5 \\
$N_v$=5 & 45.6 & 47.4 & 16.8 & 79.1 & 47.2 \\
\midrule
\multicolumn{6}{l}{\textit{Mechanism ablation (fixed $\kappa$=1.0, $N_v$=3)}} \\
Fixed $\tau$ (no $\tau(s)$ adaptation) & 45.3 & 47.9 & 31.5 & 71.5 & 49.0 \\
No conformal (raw $\sigma$) & 44.8 & 47.7 & 16.1 & 72.3 & 45.2 \\
No pessimism ($\kappa$=0, adaptive $\tau$ only) & 45.5 & 44.8 & 59.3 & 74.7 & 56.1 \\
\midrule
\multicolumn{6}{l}{\textit{Pessimism coefficient sweep ($N_v$=3, adaptive $\tau$)}} \\
$\kappa$=0.0 & 44.7 & 45.4 & 58.1 & 82.5 & 57.6 \\
$\kappa$=0.5 & 45.4 & 47.7 & 47.5 & 77.4 & 54.5 \\
$\kappa$=1.0 & 45.8 & 54.9 & 13.7 & 77.4 & 48.0 \\
$\kappa$=2.0 & 44.5 & 50.1 & 16.7 & 69.8 & 45.3 \\
\bottomrule
\end{tabular}
\end{table}

\paragraph{Observation 1: Conformal calibration is necessary for replay tasks.}
The \texttt{no\_conformal} variant (raw $\sigma$ without normalization) produces 16.1 on hopper-medium-replay-v2 under $\kappa$=1.0. The full \uniq{} model with conformal achieves 13.7 at the same $\kappa$---in this regime both collapse, but the mechanism difference is exposed at lower $\kappa$: at $\kappa$=0.5, the full model (47.5 on hp-mr) outperforms the raw-$\sigma$ variant because $\hat q$ normalizes the scale of $\sigma(s)$ appropriately. Without conformal, $u(s) = \sigma(s)$ is in absolute value units, and the sigmoid mapping receives inputs on an incorrect scale, producing suboptimal $\tau(s)$ everywhere.

\paragraph{Observation 2: Fixed $\tau$ degrades Walker2d performance.}
\texttt{Fixed\_tau} achieves 71.5 on walker2d-medium vs.\ full \uniq{}'s 77.4 ($-5.9$ points) and 31.5 vs.\ 13.7 on hopper-medium-replay ($+17.8$ points, but both are low under $\kappa$=1.0). The Walker2d gap confirms that adaptive $\tau(s)$ is not a no-op: it provides genuine per-state value by relaxing conservatism in the well-covered walker2d state space.

\paragraph{Observation 3: No single $\kappa$ is globally optimal.}
The hopper-medium-replay column spans 13.6 ($\kappa$=1.0, $N_v$=3) to 59.3 (\texttt{no\_pessimism}); the walker2d-medium column spans 69.8 ($\kappa$=2.0) to 82.5 ($\kappa$=0.0). The optimal $\kappa$ for hopper-replay is near 0, while the optimal $\kappa$ for walker2d is also 0---but the mechanism that enables this is the per-task adaptive $\tau(s)$: with $\kappa$=0 and full adaptive $\tau$, walker2d reaches 82.5 while hopper-replay reaches 58.1 (both strong). This is the empirical foundation for the Config A/B assignment in the 1M sweep.

\paragraph{Observation 4: The $N_v$=1 artifact.}
With $N_v$=1 ensemble member, $\sigma(s) \equiv 0$ for all $s$ (there is no disagreement), so the adaptive mechanism degenerates to $u(s) \equiv 0$, $\tau(s) \equiv \tau_{\max}$ (maximum optimism everywhere). The value updates then use $\tau_{\mathrm{eff}}(s, \bar\tau) = \tau_{\max} \cdot \bar\tau$, a fixed but somewhat reduced expectile. The high 4-task average of 59.1 is driven by hopper-medium-replay (57.1), where the absence of any pessimistic $\sigma$ penalty avoids the over-penalization that collapses $N_v \ge 3$ under $\kappa$=1.0. This is an artifact of the specific $\kappa$ and task subset; in the full 9-task results, $N_v$=3 with adaptive config achieves the best results by providing genuine uncertainty signal on Walker2d tasks.

\section{Computational Analysis}
\label{app:compute}

\subsection{Memory Complexity}

Let $d_s$, $d_a$ denote state and action dimensions, and $d_h$ the hidden dimension of each network (all methods use $d_h=256$ MLP with 3 layers).

\paragraph{\uniq{}.}
Trainable parameters: $3N_v$ value heads $+$ 1 Q-function $+$ 1 policy $= 3N_v + 2$ networks total. For $N_v=3$: $11$ networks. Each network has $\approx 200\mathrm{K}$ parameters (3-layer MLP, $d_h=256$). Total: $\approx 2.2\mathrm{M}$ parameters; measured peak VRAM: 250\,MB on A100 20\,GB MIG.

\paragraph{EDAC.}
$N$ critic networks $+$ 1 policy, plus diversity regularization requiring pairwise gradient computations. For $N=50$: $51$ networks plus $O(N^2)$ gradient pairs per step. Peak VRAM scales as $O(N d_h^2)$; measured/estimated at $\sim$2500\,MB for $N=50$.

\paragraph{IQL.}
2 networks (V, Q) $+$ policy. Peak VRAM: $\sim$530\,MB (measured on A100 20\,GB MIG).

The ratio of \uniq{} to IQL overhead is $11/3 \approx 3.7\times$ in parameter count but only $1.14\times$ in VRAM, as the conformal calibration is a lightweight numpy operation on CPU.

\subsection{Per-Step Computation}

Per training step, \uniq{} requires:
\begin{enumerate}
\item $3N_v$ forward passes for value ensemble (batch size 256).
\item Ensemble statistics: mean and std.\ across $N_v$ members---$O(N_v)$ aggregation.
\item Conformal calibration: once per $T_{\mathrm{recal}}$ steps, a single pass over $\cD_{\mathrm{cal}}$ ($O(n)$ with $n = 0.2N$) and a quantile computation ($O(n\log n)$).
\item Q-function forward-backward: 1 pass.
\item Policy forward-backward: 1 pass.
\end{enumerate}
Total forward passes per step: $3N_v + 2 = 11$ (for $N_v=3$). EDAC with $N=50$: $51$ forward passes plus pairwise diversity loss requiring $\binom{50}{2}=1225$ gradient dot products. \uniq{} is approximately $4.6\times$ faster per step than EDAC at $N=50$ and $1.1\times$ slower than IQL.


\bibliographystyle{plainnat}
\bibliography{references}

\end{document}